\documentclass{new_tlp}
\usepackage{mathptmx}

\usepackage{amsmath}

\newcommand\bcmdtab{\noindent\bgroup\tabcolsep=0pt%
  \begin{tabular}{@{}p{10pc}@{}p{20pc}@{}}}
\newcommand\ecmdtab{\end{tabular}\egroup}

  \title[Paracoherent Argumentation Frameworks]
        {Paracoherent Answer Set Semantics meets Argumentation Frameworks}

  \author[G. Amendola, F. Ricca]
         {GIOVANNI AMENDOLA \and FRANCESCO RICCA\\
         University of Calabria, Rende, Italy\\
         \email{\{amendola,ricca\}@mat.unical.it}}

\jdate{May 2003}
\pubyear{2003}
\pagerange{\pageref{firstpage}--\pageref{lastpage}}
\doi{S1471068401001193}

\usepackage{amssymb}
\usepackage{subfigure,tikz}
\usetikzlibrary{arrows, automata}
\usetikzlibrary{shapes,snakes}

\newtheorem{theorem}{Theorem}
\newtheorem{proposition}{Proposition}

\newtheorem{definition}{Definition}
\newtheorem{example}{Example}

\newcommand{\naf}{\mbox{\textit{not}}\ }

\newcommand{\cS}{{\mathcal F}}

\newcommand{\AS}{AS}
\newcommand{\HT}{{HT}}
\newcommand{\SST}{\mathit{SST}}
\newcommand{\SEQ}{\mathit{SEQ}}
\newcommand{\tok}[1]{\ensuremath{{#1}^\kappa}}

\newcommand{\Ik}{I^\kappa}

\newcommand{\gap}{\mathcal{G}}
\newcommand{\mi}[1]{\ensuremath{\mathit{#1}}}
\newcommand{\mc}{\ensuremath{\mi}{mc}}
\newcommand{\sig}{\Sigma}
\newcommand{\sigk}{\sig^\kappa}
\newcommand{\p}{P}
\newcommand{\pk}{P^\kappa}

\newcommand{\bodyn}[1]{B^-(#1)}
\newcommand{\la}{\leftarrow}

\newcommand{\toht}[1]{\ensuremath{{#1}^{\HT}}}

\begin{document}
\nocite{*}

\label{firstpage}

\maketitle

  \begin{abstract}
In the last years, abstract argumentation has met with great success in AI, since it has served to capture several non-monotonic logics for AI. 
Relations between argumentation framework (AF) semantics and logic programming ones are investigating more and more.
In particular, great attention has been given to the well-known stable extensions of an AF, that are closely related to the answer sets of a logic program. 
However, if a framework admits a small incoherent part, no stable extension can be provided. To overcome this shortcoming, two semantics generalizing stable extensions have been studied, namely semi-stable and stage.  
In this paper, we show that another perspective is possible on incoherent AFs, called paracoherent extensions, as they have a counterpart in paracoherent answer set semantics. 
We compare this perspective with semi-stable and stage semantics, by showing that computational costs remain unchanged, and moreover an interesting symmetric behaviour is maintained.
Under consideration for acceptance in TPLP.
  \end{abstract}



\section{Introduction}

In the last years, abstract argumentation theory~\cite{DBLP:journals/ai/Dung95} has met with great success in AI~\cite{DBLP:journals/ai/Bench-CaponD07}, since it has served to capture several non-monotonic logics for AI.
%
Recently, relations between abstract argumentation semantics and logic programming semantics has been studied systematically in \cite{DBLP:journals/ai/Strass13,DBLP:journals/ijar/CaminadaSAD15}.
These can be highlighted  by using a well-known tool for translating Argumentation Frameworks (AFs) to logic programs~\cite{DBLP:journals/sLogica/WuCG09}. 
In particular, given an AF $F$ one can build the corresponding logic program, $P_{F}$ as follows:
For each argument $a$ in $F$, if $b_1$, $b_2$, ..., $b_m$ is the set of its defeaters,  construct the rule
$ a \leftarrow \naf b_1, \ \naf b_2, \ \ldots, \ \naf b_m$.
Intuitively, each of these rules means that an argument is accepted (inferred as true) if, and only if, all of its defeaters are rejected (false) for some reason.
It is known that 
the well-founded model of $P_{F}$~\cite{DBLP:journals/jacm/GelderRS91} corresponds to the grounded extension of $F$~\cite{DBLP:journals/ai/Dung95}; 
stable models (or answer sets) of $P_{F}$~\cite{DBLP:journals/ngc/GelfondL91} correspond to stable extensions of $F$~\cite{DBLP:journals/ai/Dung95}; 
regular models~\cite{DBLP:journals/jcss/YouY94} correspond to preferred extensions; P-stable models~\cite{DBLP:journals/ai/Przymusinski91} correspond to complete extensions~\cite{DBLP:journals/sLogica/WuCG09}; and L-stable models~\cite{DBLP:journals/amai/EiterLS97} correspond to semi-stable extensions~\cite{DBLP:journals/ijar/CaminadaSAD15}.
%
%
Focusing on the stable semantics, we recall that a stable extension is defined as a conflict-free set (i.e., no argument in the set attacks another one in the set) that attacks every other argument outside of it~\cite{DBLP:journals/ai/Dung95}.
In particular, the labelling-approach~\cite{DBLP:journals/logcom/JakobovitsV99} implies that no argument is labelled as undecidable (i.e., an argument will be either true or false).
This condition yields very solid solutions.
For instance, consider the AF reported in Figure~\ref{fig:noStable}(a). 
As both arguments $a$ and $c$ are attacked by no other argument, they will be true; so that, $b$ will be false (it is attacked by $a$ and $c$); and, finally, $d$ will be true, as it is attacked by $b$ only. Hence, $\{a,c,d\}$ is a stable extension.
However, this solidity is given at the price of not offering any solution in many situations. 
In particular, if the AF includes even a small part admitting no stable extension, no stable extension can be provided for the entire AF. This is explained by saying that this semantics is not ``crash resistant''~\cite{DBLP:journals/logcom/CaminadaCD12}.
For instance, if we consider the AF reported in Figure~\ref{fig:noStable}(b), that differs from the previous one only for an attack from $a$ to itself, then no stable extension exists. 


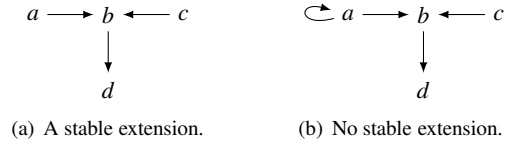
\begin{figure}
\centering
{
\subfigure[A stable extension.]
{
\begin{tikzpicture}[scale=1.0, transform shape]

\node[] (a) at (0:0) {$a$};
\node[] (f) at (0:-0.5) {};
\node[] (e) at (0:2.5) {};
\node[] (b) [right of= a] {$b$};
\node[] (c) [right of= b] {$c$};
\node[] (d) [below of= b] {$d$};

\path[->, >=latex] (a)	edge node {} (b)
					(c)	edge node {} (b)
					(b)	edge  node {} (d);
\end{tikzpicture}
}
\hspace*{0.5cm}
\subfigure[No stable extension.]
{
\begin{tikzpicture}[scale=1.0, transform shape]

\node[] (a) at (0:0) {$a$};
\node[] (b) [right of= a] {$b$};
\node[] (c) [right of= b] {$c$};
\node[] (d) [below of= b] {$d$};

\path[->, >=latex] (a)	edge [loop left] node {} (a)
						edge node {} (b)
					(c)	edge node {} (b)
					(b)	edge  node {} (d);
\end{tikzpicture}
}
}
\caption{Examples of argumentation frameworks.}
\label{fig:noStable}
\end{figure}

To overcome this shortcoming, argumentation semantics that extend the stable one have been proposed. More specifically, 
they coincide with the stable semantics, whenever a stable extension exists; and provide others solutions, whenever no stable extension exists. 
Currently, there are two semantics with these properties:
the Stage semantics~\cite{Verheij96twoapproaches} and
the Semi-stable semantics~\cite{DBLP:conf/comma/Caminada06}.
A stage extension is a conflict-free set of arguments $A$, such that $A\cup A^{+}$ is (subset) maximal with respect to conflict-free sets, where $A^{+}$ denotes the set of all arguments attacked by an argument of $A$.
A semi-stable extension is a 
conflict-free set $A$, where each argument attacking an argument of $A$, is in turn attacked by an argument of $A$, such that $A\cup A^{+}$ is maximal.
By considering the AF reported in Figure~\ref{fig:noStable}(b), the set $\{c,d\}$ is both a stage extension and a semi-stable extension. However, in general, stage semantics and semi-stable semantics are very different~\cite{DBLP:journals/logcom/CaminadaCD12}.

In this paper, we provide an alternative view on how to generalize the stable semantics, based on paracoherent semantics introduced in Answer Set Programming (ASP)~\cite{DBLP:journals/cacm/BrewkaET11}.
Actually, there are two main paracoherent semantics for logic programs. The first was introduced by~\cite{DBLP:journals/logcom/SakamaI95}, and it is known as \textit{semi-stable model semantics} (not to be confused with semi-stable semantics in argumentation); while the most recent was introduced by~\cite{DBLP:conf/kr/EiterFM10} to avoid some anomalies concerning modal logic properties, and it is known as \textit{semi-equilibrium model semantics}.
As we will see, these two semantics coincide in our settings, and we refer to them as \textit{paracoherent answer sets}. 

\begin{figure}
\centering 
\subfigure[$F_a$]
{
\begin{tikzpicture}[scale=1.0, transform shape]
\def \n {4}
\def \radius {1cm}
\def \radiusExt {2cm}
\def \margin {8} 
\def \angle {0}

\node[] (c) at (0:0) {$s$};
\node[] (b1) at ({360/\n * (1 )+\angle}:\radius) {$js$};
\node[] (b2) at ({360/\n * (2 )+\angle}:\radius) {$ms$};
\node[] (b3) at ({360/\n * (3 )+\angle}:\radius) {$as$};
\node[] (b4) at ({360/\n * (4 )+\angle}:\radius) {$rs$};
\node[] (a1) at ({360/\n * (1 + 0.5)+\angle}:\radiusExt) {$jm$};
\node[] (a2) at ({360/\n * (2 + 0.5)+\angle}:\radiusExt) {$am$};
\node[] (a3) at ({360/\n * (3 + 0.5)+\angle}:\radiusExt) {$ar$};
\node[] (a4) at ({360/\n * (4 + 0.5)+\angle}:\radiusExt) {$jr$};

\foreach \s in {1,...,\n}
{

  \path[->, >=latex] ({360/\n * (\s + 0.5)+\angle+\margin}:1.8) 
    edge ({360/\n * (\s + 1.5)+\angle-\margin}:1.8);


  \path[->, >=latex]	({360/\n * (\s - 0.5)+\angle}:\radiusExt-0.3cm)  edge ({360/\n * (\s - 0.8)+\angle}:0.2cm+\radius);
  \path[->, >=latex]	({360/\n * (\s - 0.5)+\angle}:\radiusExt-0.3cm)  edge ({360/\n * (\s - 0.2)+\angle}:0.2cm+\radius);

  \path[->, >=latex]	({360/\n * (\s - 1)+\angle}:-0.2cm+\radius-2)  edge (c);
}
\end{tikzpicture}
}
\hspace*{0.5cm}
%
%
%
%
%
%
\subfigure[$F_b$]
{
\begin{tikzpicture}[scale=1.0, transform shape]
\def \n {3}
\def \radius {0.7cm}
\def \radiusExt {1.7cm}
\def \margin {10} 

\node[] (c) at (0:0) {$s$};

\node[] (b1) at ({360/\n * (1 )}:\radius) {$js$};
\node[] (b2) at ({360/\n * (2 )}:\radius) {$ms$};
\node[] (b3) at ({360/\n * (3 )}:\radius) {$as$};

\node[] (a1) at ({360/\n * (1 +0.5)}:\radiusExt) {$jm$};
\node[] (a2) at ({360/\n * (2 +0.5)}:\radiusExt) {$am$};
\node[] (a3) at ({360/\n * (3 +0.5)}:\radiusExt) {$aj$};

\foreach \s in {1,...,\n}
{

  \path[->, >=latex] ({360/\n * (\s + 0.5)+\margin}:\radiusExt-0.2cm) 
    edge ({360/\n * (\s + 1.5)-\margin}:\radiusExt-0.2cm);

  \path[->, >=latex]	({360/\n * (\s - 0.5)}:\radiusExt-0.3cm)  edge ({360/\n * (\s - 0.8)}:0.05cm+\radius-0.5);
  \path[->, >=latex]	({360/\n * (\s - 0.5)}:\radiusExt-0.3cm)  edge ({360/\n * (\s - 0.2)}:0.05cm+\radius-0.5);

  \path[->, >=latex]	({360/\n * (\s - 1)}:-0.2cm+\radius)  edge (c);
}
\end{tikzpicture}
}
\caption{Argumentation frameworks for SRP.}
\label{fig:roomates}
\end{figure}
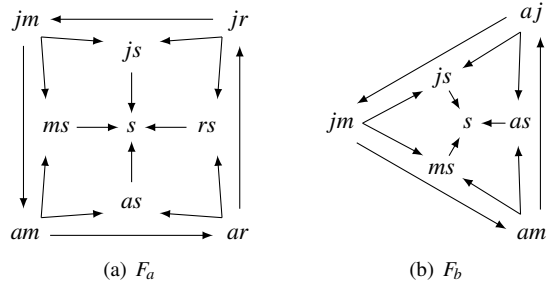

A distinctive property of the new paracoherent semantics for AF is a symmetric behaviour, that is not present in stage and semi stable semantics.
To exemplify it, we use a concrete classic application of AF considered orginally by~\cite{DBLP:journals/ai/Dung95}, namely the \textit{Stable Roommates Problem} (SRP). 
SRP is a matching problem with several variants and real-word applications, where persons have to be matched while respecting their preferences~\cite{marriage}. 
Suppose that 
 Mark ($m$) prefers John ($j$) to Annie ($a$), and Annie to Shrek ($s$); 
 John prefers  Robert ($r$) to Mark, and Mark to Shrek; 
 Robert prefers Annie to John, and John to Shrek; 
 Annie prefers Marc to Robert, and Robert to Shrek; 
 while Shrek wants to stay alone.
The AF $F_{a}$ in Figure~\ref{fig:roomates}(a) models the problem using arguments for possible matchings that attack each other according to preferences. 
Note that $F_{a}$ admits two stable extensions (i.e., matchings).
One solution pairs John and Mark ($jm$), Annie and Robert ($ar$), while Shrek stays alone; and the other solution pairs Annie and Mark ($am$), John and Robert ($jr$), again Shrek stays alone.
Hence ``Shrek is alone'' is a skeptical argument, as expected.
Consider now the AF $F_b$ in Figure~\ref{fig:roomates}(b) with one less person and the following preferences: Mark prefers John to Annie, and Annie to Shrek; John prefers Annie  to Mark, and Mark to Shrek; Annie prefers Mark  to John, and John to Shrek; again Shrek wants to stay alone. 
Despite the AF $F_{b}$ has the same symmetric structure of $F_{a}$, it has no stable extension (due to odd-length cycles). This means that the SRP problem has no solution, and, in practice, one has to accept a non comfortable solution (i.e., a non stable one). 
Pragmatically, one can minimizes the number of persons that have no roommate, and always keep Shrek (disliking all mates) alone, as in the stable case. 
Unluckily, the semi-stable semantics cannot suggest a practical solution to the problem (as the unique admissible set is the empty one); 
on the other hand, it can be verified that in stage extensions Shrek is always paired. 
Thus, both semantics extending the stable one do not behave as in $F_a$, where Shrek is always alone and persons that have no roommate are minimized, i.e., we would expect to pair Mark with John or John with Annie or Annie with Mark.
This symmetric behavior is kept in the semantics introduced here.
As argued in~\cite{DBLP:journals/ai/BaroniGG05a} ``it is counter-intuitive that different results in conceptually similar situations depend on the length of the cycle: symmetry reasons suggest that all cycles should be treated equally and should yield the same results".
According to this observation, our semantics provides an approach that is more adherent to this ideal behavior than related proposals.
Moreover, paracoherent extensions could represent plasible scenarios in several real world applications of AFs where odd cycles naturally appear, such as legal reasoning~\cite{DBLP:journals/ail/PrakkenS96,DBLP:journals/ai/Verheij03,DBLP:journals/logcom/Gabbay16a}, dialog and negotiation~\cite{DBLP:conf/ecai/AmgoudPM00}, planning~\cite{DBLP:journals/ai/Pollock98}, and traffic networks~\cite{thanksMarco}.

In the paper, we define the \textit{paracoherent extensions} for AFs that are based on the concept of \textit{stabilizer}. Stabilizers capture what is missing to a set of arguments to become stable.
Next, we show that paracoherent extensions can be considered as generalization of the stable extensions, just like the semi-stable and the stage ones.
Then, by using the direct translation from AFs to logic programs, we show a correspondence between paracoherent extensions and paracoherent answer sets.
Moreover, we prove that credulous and skeptical reasoning for paracoherent semantics  are $\Sigma_2^P$-complete and $\Pi_2^P$-complete, respectively (as for semi-stable and stage~\cite{DBLP:books/sp/09/DunneW09,DBLP:conf/jelia/DunneC08,DBLP:journals/ipl/DvorakW10}).
Finally, we discuss and relate the paracoherent semantics with several existing ones, showing that paracoherent semantics has an interesting symmetric behaviour on graphs with odd-length cycles.



\section{Preliminaries}\label{Sec:prel}
We overview basic concepts in argumentation and paracoherent answer set semantics.
\subsection{Abstract Argumentation Semantics}

An AF $F$ is a pair $( Ar,att )$, where $Ar$ is a finite set of \textit{arguments}, and $att\subseteq Ar\times Ar$ is a set of \textit{attacks}. Hence, an AF can be represented by a directed graph, where nodes are arguments, and edges are attacks. 
For instance, concerning the AF in Figure~\ref{fig:noStable}(a), we have $Ar=\{a,b,c,d\}$, and $att=\{(a,b),(c,b),(b,d)\}$.
Let $A\subseteq Ar$ be a set of arguments.
We denote by $A^{+}$ the set of all arguments in $Ar$ attacked by an argument in $A$, i.e. $A^{+}=\{b\in Ar \mid (a,b)\in att\mbox{, and } a\in A\}$.
Then, $A$ is \textit{conflict-free} in $F$ if, for each $a,b\in A$, $(a,b)\not\in att$; 
$A$ is \textit{admissible} in $F$ if $A$ is conflict-free, and, given $a\in A$, for each $b\in Ar$ with $(b, a)\in att$, there is $c \in A$ such that $(c,b)\in att$ (say that $a$ is \textit{defended} by $A$ in $F$);
$A$ is \textit{complete} in $F$ if $A$ is admissible in $F$ and each $a\in Ar$ defended by $A$ in $F$ is contained in $A$;
$A$ is a \textit{stable extension} if $A^{+} = Ar\setminus A$ (i.e., $A$ is conflict-free, and 
for each $a \in Ar\setminus A$, there is $b\in A$, such that $(b, a)\in att$);
$A$ is a \textit{semi-stable extension} if $A$ is complete and $A\cup A^{+}$ is maximal
(hereafter, we write just maximal for maximal w.r.t. subset inclusion);
$A$ is a \textit{stage extension} if $A$ is conflict-free and, for each $B$ conflict-free, $A\cup A^{+}\not\subset B\cup B^{+}$  (i.e., $A\cup A^{+}$ is maximal with respect to conflict-free sets).
We denote by $\mathit{cf}(F)$, $\mathit{adm}(F)$, $\mathit{comp}(F)$, $\mathit{stb}(F)$, $\mathit{sem}(F)$, $\mathit{stage}(F)$, the sets of all conflict-free, admissible, complete, stable, semi-stable, and stage extensions, respectively.

\begin{example}\label{Ex:general}
Let $F=(\{a,b,c,d\},\{(a,b),(c,b),(b,d)\}$ be the AF reported in Figure~\ref{fig:noStable}(a).
We have $\mathit{cf}(F)=\{\emptyset,$ $\{a\},$ $\{b\},$ $\{c\},$ $\{d\},$ $\{a,c\},$ $\{a,d\},$ $\{c,d\},$ $\{a,c,d\} \}$;
$\mathit{adm}(F)=\{\emptyset,$ $\{a\},$ $\{c\},$ $\{a,c\},$ $\{a,d\},$ $\{c,d\},$ $\{a,c,d\} \}$;
$\mathit{comp}(F)=\{\{a,c,d\} \}$;
$\mathit{stb}(F)=\mathit{sem}(F)=\mathit{stage}(F)=\{ \{a,c,d\} \}$.
\end{example}

\noindent It is known that $\mathit{cf}(F)$ $\supseteq$ $\mathit{adm}(F)$ $\supseteq$ $\mathit{comp}(F)$ $\supseteq$ $\mathit{sem}(F)$ $\supseteq$ $\mathit{stb}(F)$; while 
$\mathit{cf}(F)$ $\supseteq$ $\mathit{stage}(F)$ $\supseteq$ $\mathit{stb}(F)$. 
Moreover, both semi-stable and stage semantics coincide with the stable one, whenever a stable extension exists; and provide others solutions, whenever no stable one exists~\cite{DBLP:journals/logcom/CaminadaCD12}. 
The main difference between the two semantics is that, on the one hand, a stage extension is a maximal conflict-free set, while in general the semi-stable extension is not; 
on the other hand, a semi-stable extension is an admissible extension, while in general the stage extension is not.
\begin{example}\label{Ex:Stage-Semi}
The AF reported in Figure~\ref{fig:stage-semi} has no stable exstension, while it has a unique semi-stable extension, $\{a,d\}$; and 3 stage ones $\{a,c,e\}$, $\{a,c,g\}$, and $\{a,d,g\}$.
\end{example}
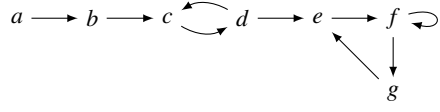
\begin{figure}
\centering
{
\begin{tikzpicture}[scale=1.0, transform shape]

\node[] (a) at (0:0) {$a$};
\node[] (b) [right of= a] {$b$};
\node[] (c) [right of= b] {$c$};
\node[] (d) [right of= c] {$d$};
\node[] (e) [right of= d] {$e$};
\node[] (f) [right of= e] {$f$};
\node[] (g) [below of= f] {$g$};

\path[->, >=latex] (a)	edge node {} (b)
					(b)	edge node {} (c)
					(c)	edge [bend right] node {} (d)
					(d)	edge [bend right] node {} (c)
						edge node {} (e)
					(e)	edge node {} (f)
					(f)	edge [loop right] node {} (f)
						edge node {} (g)
					(g) edge node {} (e);
\end{tikzpicture}
}
\caption{Argumentation framework of Examples~\ref{Ex:Stage-Semi},~\ref{ex:para-sem-stage},~\ref{ex:PF},~\ref{ex:para=seq}.}
\label{fig:stage-semi}
\end{figure}


Moreover, we also recall the following decision problems corresponding to reasoning tasks on AFs.
Given an AF $F$, a semantics $\sigma$, and an argument $a$, decide whether $a$: 
$(1)$ is contained in some $\sigma$-extension of $F$ (credulous reasoning);
$(2)$ is contained in all $\sigma$-extensions of $F$ (skeptical reasoning).
It is known that credulous reasoning is $\mathsf{NP}$-complete for stable semantics, and $\Sigma_2^P$-complete for semi-stable and stage semantics; 
while skeptical reasoning is $\mathsf{coNP}$-complete for stable semantics, and $\Pi_2^P$-complete for semi-stable and stage semantics~\cite{DBLP:books/sp/09/DunneW09,DBLP:conf/jelia/DunneC08,DBLP:journals/ipl/DvorakW10}.

\subsection{Paracoherent Answer Set Semantics}

We concentrate on logic programs over a propositional signature $\Sigma$.
A \textit{rule} $r$ is of the form 
\begin{equation}\label{eq:rule}
a_1 \vee\ldots\vee a_l \leftarrow b_1,\ldots,b_m, \naf c_1,\ldots,\naf c_n
\end{equation} 
\noindent where all $a_i$, $b_j$ and $c_k$ are atoms (from $\Sigma)$; 
$l,m,n\geq 0$, and $l+m+n>0$; 
$\naf$represents \textit{default negation}. 
The set $H(r)=\lbrace a_{1},...,a_{l} \rbrace$ is the \textit{head} of $r$, 
while $B^{+}(r)=\lbrace b_{1},...,b_{m} \rbrace$ and $B^{-}(r)=\lbrace
c_{1},\ldots,c_{n} \rbrace$ are
the \textit{positive body} and the \textit{negative body} of $r$,
respectively;
the \textit{body} of $r$ is $ B(r)=B^{+}(r)\cup B^{-}(r)$. 
If $B(r)=\emptyset$, we then omit $\leftarrow$; 
and if 
$|H(r)| \leq 1$,
then $r$ is \textit{normal}. 
A \textit{program} $P$ is a
finite set of rules. 
$P$ is called 
\textit{normal} 
if each $r\in P$ is 
normal. 

Any set $I\subseteq \Sigma$ is an \textit{interpretation}.
$I$ is a \textit{model} of a program $P$ (denoted
$I\models P$) iff for each rule $r\in P$, $I\cap H(r)\neq \emptyset$ whenever
$B^{+}(r)\subseteq I$ and $B^{-}(r)\cap I=\emptyset$ (denoted $I \models
r$).  
A model $M$ of $P$ is \textit{minimal} iff no model $M'\subset M$ of $P$ exists.
Given an interpretation $I$, we denote by
$P^I$ the well-known \textit{Gelfond-Lifschitz reduct} \cite{DBLP:journals/ngc/GelfondL91} of
$P$ w.r.t. $I$,
that is the set of rules
$ a_{1}\vee ...\vee a_{l} \leftarrow b_{1},...,b_{m}$,
obtained from rules $r\in P$ of form (\ref{eq:rule}), such that
$B^{-}(r)\cap I=\emptyset$.
A model $M$ of $P$ is called {\em answer set} (or {\em stable model}) of $P$, if $M$ is a minimal model of $P^M$.	 
We denote by $AS(P)$ the set of all answer sets of $P$.
%
%
Next, we introduce two paracoherent semantics.
The first one is known as \emph{semi-stable model semantics}~\cite{DBLP:journals/logcom/SakamaI95}.
We consider an extended signature
$\sigk = \sig\cup\{Ka\mid a\in \sig\}$.
Intuitively, $Ka$ can be read as $a$ is believed to hold. 
The semi-stable models of a program $\p$ are obtained from its \emph{epistemic $\kappa$-transformation}. 
\begin{definition}
[Epistemic $\kappa$-transformation $P^{\kappa}$]\label{def:epi-trans}
Let $P$ be a program. Then its epistemic $\kappa$-transformation is defined as 
the program $P^{\kappa}$ obtained from $P$ by replacing each rule $r$  
of the form~(\ref{eq:rule}) in $P$, such that $\bodyn{r}\neq\emptyset$, with:
\begin{align}
\lambda_{r,1} \vee \ldots \vee \lambda_{r,l} \vee Kc_{1} \vee \ldots \vee Kc_n & \la  b_1,\ldots, b_m; \label{eq:ep-1}  \\
a_i & \la \lambda_{r,i}; \label{eq:ep-2} \\
    & \la \lambda_{r,i}, c_j; \label{eq:ep-3} \\
\lambda_{r,i} & \la  a_i, \lambda_{r,k}; \label{eq:ep-4}
\end{align}
for $1\leq i,k\leq l$ and $1\leq j\leq n$, where $\lambda_{r,i}$ are fresh atoms.
\end{definition}
Given an interpretation $\Ik$ over $\sig'\supseteq\sigk$, let $\gap(\Ik)=\{ Ka\in\Ik\ \mid a\not\in\Ik\}$ denote the atoms believed true but not assigned true, also referred to as the gap of $\Ik$.
Given a set $\cS$ of interpretations over $\sig'$, an interpretation $\Ik\in \cS$ is \emph{maximal canonical in $\cS$}, if no $\tok{J}\in \cS$ exists such that 
$\gap(\Ik)\supset\gap(\tok{J})$.
By $\mc(\cS)$ we denote the set of maximal canonical interpretations in $\cS$.
Semi-stable models are defined as maximal canonical interpretations among the answer sets of $\pk$, and
the set of all semi-stable models of $\p$ is denoted by $\SST(\p)$,
i.e., $\SST(\p)= \{  S \cap \sigk \mid S\in\mc(\AS(\pk))\}$.

The second one is called \textit{semi-equilibrium model semantics} and was introduced by~\cite{DBLP:conf/kr/EiterFM10}
to amend anomalies in semi-stable model semantics.
Semi-equilibrium models may be computed as maximal canonical 
answer sets of an extension of the epistemic $\kappa$-transformation. 

\begin{definition}[Epistemic $\HT$-transformation $\toht{\p}$]\label{def:ht-trans}
Let $\p$ be a program over $\sig$. 
Then its epistemic $\HT$-transformation 
$\toht{\p}$ is defined as the union of $\pk$ with the set of rules:

\begin{align}
Ka & \la a, \label{eq:ep-5}\\
Ka_1 \vee \ldots \vee Ka_l \vee Kc_{1} \vee \ldots \vee Kc_n 
&\la  Kb_1,\ldots, Kb_m,  \label{eq:ep-6}
\end{align}
\noindent for $a \in \sig$, respectively for every rule $r\in\p$ of the form~(\ref{eq:rule}).
\end{definition}
Then, the set of all semi-equilibrium models is given by
$\{ M\cap\sigk \mid M\in\mc(\AS(\toht{\p}))\}$ and is denoted by 
$\SEQ(\p)$.
In the following, we refer to semi-stable models or semi-equilibrium models as \textit{paracoherent answer sets}, and
we will be interested to consider only the true atoms of each paracoherent answer set. Hence, we denote by $\SST^{\bf t}(P)$ the set $\{  M \cap \sig \mid M\in \SST(P)\}$, and by $\SEQ^{\bf t}(P)$, the set $\{ M\cap\sig \mid M\in \SEQ(P)\}$.

\section{Paracoherent Extensions}

In this section, we formally introduce the paracoherent semantics for argumentation.
Our goal is to provide another reasonable generalization of the stable semantics. 
Intuitively, we start identifying what is missing to a candidate set of arguments to become stable.
This intuition is formalized by the basic concept of stabilizer.

\begin{definition}[Stabilizer]
Let $F = (Ar,att)$ be an AF. A set of arguments $S\subseteq Ar$ is called a \textit{stabilizer}, if there exists $A\subseteq Ar$, such that  $A^+\cup S^+ = Ar\setminus A$.
We will say that \textit{$S$ is a stabilizer of $A$}, and also that \textit{$A$ admits $S$ as stabilizer}.
\end{definition}

\begin{example}\label{ex:caut}
The AF reported in Figure~\ref{fig:1},
has the empty set as the unique admissible set of arguments.
Now, let $A=\{a\}$. Since $a$ attacks $b$ and $d$, we have that $A^+=\{b,d\}$. Therefore, a set of arguments $S$ to be a stabilizer of $A$ needs to attack at least arguments $c$ and $e$, and is not allowed to attack $a$. 
So that, $S=\{b,d\}$ is a stabilizer of $A$.
Indeed $A^+\cup S^+$ $=$ $\{b,$ $c,$ $d,$ $e\}$ $=$ $Ar\setminus A$.
Similarly, it can be seen that $\{a,b,d\}$, $\{b,d,e\}$ and $\{a,b,d,e\}$ are the remaining stabilizers of $A$.
\end{example}

Note that, in general, there are sets of arguments that do not admit any stabilizer. E.g., by considering  the AF in Example~\ref{ex:caut}, then the set $A'$ $=$ $\{a,b\}$ is such. This is due to the fact that $a$ attacks $b$, thus $b\in A^+ \subseteq (A^+\cup S^+)$, for each $S\subseteq Ar$.  In general, we prove the following results.

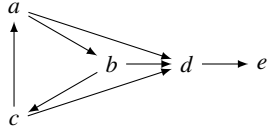
\begin{figure}
\centering
{
\begin{tikzpicture}[scale=1.0, transform shape]

\node[] (a) at (-30:-1.5) {$a$};
\node[] (b) at (0:0) {$b$};
\node[] (c) at (30:-1.5) {$c$};
\node[] (d) [right of= b] {$d$};
\node[] (e) [right of= d] {$e$};
\path[->, >=latex] (b)	edge node {} (d)
						edge node {} (c)
					(a)	edge node {} (d)
						edge node {} (b)
					(c)	edge node {} (d)
						edge node {} (a)
					(d)	edge node {} (e);
\end{tikzpicture}
}
\caption{Argumentation framework of the Examples~\ref{ex:caut},~\ref{ex:paraExtension}.}
\label{fig:1}
\end{figure}

\begin{proposition}\label{thm:properties}
Let $F = (Ar,att)$ be an AF, and let $S,A\subseteq Ar$ be two sets of arguments. Then,
$(1)$ if $S$ is a stabilizer of $A$, then $A$ is conflict-free;
$(2)$ if $S$ is a stabilizer of $A$, then $A \cap S^+=\emptyset$; and
$(3)$ $\emptyset$ is a stabilizer of $A$ iff $A$ is stable.
\end{proposition}

%
%

To define argumentation semantics as close as possible to the stable extension, we need to select sets of arguments that admit stabilizers of minimal size. Thus, given an AF $F=(Ar,att)$, let $\Sigma_F = \{S \mid \exists A\subseteq Ar \mbox{ s.t. } S \mbox{ is a stabilizer of } A\}$ be the set of all possible stabilizers.

\begin{definition}[Paracoherent Extension]
Let $F = (Ar,att)$ be an AF. A set of argument $A\subseteq Ar$ is a \textit{paracoherent extension} of $F$, if $A$ admits a minimal stabilizer among $\Sigma_F$ w.r.t. $\subseteq$.
\end{definition}

\begin{example}\label{ex:paraExtension}
Consider the AF of Example~\ref{ex:caut}.
Note that $\emptyset$ is not a stabilizer, since by Proposition~\ref{thm:properties} $F$ does not admit any stable extension. However, for instance, $A=\{a,e\}$ is a paracoherent extension.
Indeed, $S=\{b\}$ is a minimal stabilizer of $A$, as $(A^+\cup S^+) = \{a,b,e\}^+ = \{b,c,d\} = Ar\setminus A$.
\end{example}

Now, we show formally that the introduced argumentation semantics behaves as an extension of the stable argumentation one such as the semi-stable and the stage semantics.

\begin{theorem}\label{thm:extended}
Let $F = (Ar,att)$ be an AF. 
If $A\subseteq Ar$ is a stable extension, then $A$ is also a paracoherent extension.
\end{theorem}
\begin{proof}
Let $A$ be a stable extension. Then, by Proposition~\ref{thm:properties}(3), $\emptyset$ is a stabilizer of $A$. As $\emptyset$ is minimal among all stabilizers in $\Sigma_F$ , then $A$ is a paracoherent extension.
\end{proof}

\begin{theorem}\label{thm:coincidence}
Let $F = (Ar,att)$ be an AF. 
If $\mathit{stb}(F) \neq \emptyset$, then $\mathit{para}(F)=\mathit{stb}(F)$.
\end{theorem}
\begin{proof}
Assume that $\mathit{stb}(F) \neq \emptyset$.
Hence, there is a stable extension $A$. Thus, by Theorem~\ref{thm:extended}, $A$ is a paracoherent extension.
Now, let $A'$ be a paracoherent extension. Since by assumption $A$ is a stable extension, we know, by Proposition~\ref{thm:properties}(3), that $\emptyset$ is a stabilizer of $A$. Hence, $\emptyset\in\Sigma_F$. Thus, $\emptyset$ is the unique minimal stabilizer. Therefore, it is also the minimal stabilizer of $A'$.
Hence, again by Proposition~\ref{thm:properties}(3), $A'$ is a stable extension.
\end{proof}

However, paracoherent semantics differs from both semi-stable semantics and stage semantics.
\begin{example}\label{ex:para-sem-stage}
Let $F$ be the AF reported in Figure~\ref{fig:stage-semi}. 
In Example~\ref{Ex:Stage-Semi}, we have seen that
$\mathit{sem}(F)=\{\{a,d\}\}$, and $\mathit{stage}(F)$ $=$ $\{\{a,c,e\},$ $\{a,c,g\},$ $\{a,d,g\}\}$.
Now, it can be checked that
$\mathit{para}(F)$ $=$ $\{\{a,d\},$ $\{a,c,e\},$ $\{a,c,g\},$ $\{a,d,g\}\}$.
\end{example}
As it will be clearer in Section~\ref{sec:related} paracoherent, semi-stable and stage are actually incomparable.

In Figure~\ref{fig:taxonomy}, a taxonomy of the argumentation semantics cited in the paper is reported. An arrow from a semantics $\sigma$ to a semantics $\sigma'$ means that $\sigma(F)\subseteq \sigma'(F)$, for each AF $F$, and there is an AF $F'$ such that $\sigma(F)\subset \sigma'(F)$.
Note that the paracoherent semantics is not admissible-based.
E.g., the paracoherent extension $\{a,e\}$ of the Example~\ref{ex:paraExtension} is not admissible, as the only admissible set of $F$ is the empty set. 
Note that, the non-admissibility property is also shared by the \textit{stage} semantics.

\begin{figure}
\centering
{
\begin{tikzpicture}[scale=1.0, transform shape]
\tikzstyle{m}=[circle, thin, minimum size=8mm,inner sep=0pt]

\node[] (a) at (0:0)   {$\mathit{cf}$};
\node[] (b) at (0:1.5) {$\mathit{adm}$};
\node[] (c) at (0:3)   {$\mathit{cmp}$};
\node[] (d) at (0:4.5) {$\mathit{sem}$};
\node[] (e) at (0:6)   {$\mathit{stb}$};
\node[] (f) at (9:4.5)   {$\mathit{stage}$};
\node[] (g) at (-9:4.5)  {$\mathit{para}$};

\path[->, >=latex] (b)	edge node {} (a)
					(c)	edge node {} (b)
					(d)	edge node {} (c)
					(e)	edge node {} (d)
					(f)	edge [bend right=13] node {} (a)
					(e)	edge node {} (f)
					(g)	edge [bend left=13] node {} (a)
					(e)	edge node {} (g);
\end{tikzpicture}
}
\caption{Taxonomy of some argumentation semantics.}
\label{fig:taxonomy}
\end{figure}
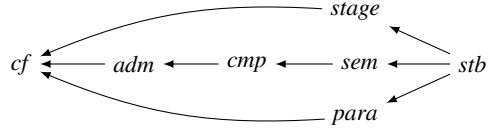

\section{Relation with Logic Programming}
In this section, we study the relation of our semantics for AFs with paracoherent semantics for logic programs.
In particular, we show that paracoherent extensions of an AF $F$ correspond to paracoherent answer sets of a logic program associated to $F$~\cite{DBLP:journals/sLogica/WuCG09}.

\begin{definition}
Let $F=(Ar,att)$ be an AF. For each argument $a\in Ar$, we build a rule $r_a$ such that $H(r_a)=\{a\}$, $B^{+}(r_a)=\emptyset$, and $B^{-}(r_a)=\{c\in Ar\mid (c,a)\in att\}$.
Then, we define $P_F$ as the set 
$\{r_a \mid a\in Ar\}$.
\end{definition}

\begin{example}\label{ex:PF}
Consider the AF $F$ reported in Figure~\ref{fig:stage-semi}. Its corresponding logic program is
$P_{AF}=\{a;\ 
b \leftarrow \naf a;\
c \leftarrow \naf b, \ \naf d;\
d \leftarrow \naf c;\
e \leftarrow \naf d, \ \naf g;\
f \leftarrow \naf e, \ \naf f;\
g \leftarrow \naf f \}$.

\end{example}

Since there are two main paracoherent ASP semantics in logic programming, semi-stable model semantics and semi-equilibrium model semantics (see Preliminaries),
we start to show that on logic programs having the particular structure of $P_F$, the two semantics coincide. 


\begin{theorem}\label{th:sst=seq}
Let $P_F$ be the logic program corresponding to an AF $F$.
Then, $\mathit{SST}^{\bf t}(P_F)=\mathit{SEQ}^{\bf t} (P_F)$.
\end{theorem}
\begin{proof}
First, the epistemic $\kappa$-transformation of $P_F$, $P_F^\kappa$, has a simplified structure. As each rule $r$ in $P_F$ is normal and it has an empty positive body (i.e. $r$ is of the form $a\leftarrow \naf c_1,..., \naf c_n$), then
rule~(\ref{eq:ep-1}) becomes $\lambda_r \vee Kc_1 \vee ... \vee Kc_n$;   
rule~(\ref{eq:ep-2}) becomes $a\leftarrow \lambda_r$;
rule~(\ref{eq:ep-3}) becomes $\leftarrow \lambda_r, c_j$;
while rule~(\ref{eq:ep-4}) becomes irrelevant (it is $\lambda \leftarrow a,\lambda$), hence it can be removed.
Moreover, as in $P_F$ there is a unique rule having $a$ in the head, then rules~(\ref{eq:ep-1})-(\ref{eq:ep-2}) can be unified in the rule $a \vee Kc_1 \vee ... \vee Kc_n$, and rule~(\ref{eq:ep-3}) is equivalent to $\leftarrow a,c_j$. Thus, 
the program $P_F^{\kappa}$ is obtained by replacing each rule $r$, such that $\bodyn{r}\neq\emptyset$, with:
$(8)$ $a \vee Kc_1 \vee ... \vee Kc_n$ and 
$(9)$ $\la a, c_j$,   for $j=1,...,n$.
%
Concerning the epistemic $\HT$-transformation of $P_F$, $P_F^\HT$, rule~(\ref{eq:ep-6}) becomes $Ka \vee Kc_1 \vee ... \vee Kc_n$. However, this rule can be derived by rule~(8) and rule~(\ref{eq:ep-5}). Hence, it can be removed.   
Therefore, the program $P_F^{\HT}$ is obtained by
adding to $P_F^\kappa$, only rules of the form~(\ref{eq:ep-5}), i.e., $Ka \la a$. 
Now, it is easy to see that if $M$ is a minimal model of $P^\kappa_F$, then $M\cup\{Ka \mid a\in M \mbox{ and } Ka\not\in M\}$ is a minimal model of $P^\HT_F$; and if $M'$ is a minimal model of $P^\HT_F$, then there is $M\subseteq M'$, where eventually some $Ka$ atoms are removed, such that $M$ is a minimal model of $P^\kappa_F$.
As both maps do not change the gap of the models and the true atoms, then 
$\{ M\cap\sig \mid M\in\mc(\AS(P^\kappa_F))\}$ $=$
$\{ M\cap\sig \mid M\in\mc(\AS(P^\HT_F))\}$, i.e., $\mathit{SST}^{\bf t}(P_F)=\mathit{SEQ}^{\bf t} (P_F)$.
%
%
\end{proof}

Paracoherent extensions of an AF $F$ coincide with paracoherent answer sets of $P_F$.


\begin{theorem}\label{th:para=seq}
Let $P_F$ be the logic program corresponding to an AF $F$.
Then, $\mathit{para}(F)=\mathit{SEQ}^{\bf t}(P_F)$.
\end{theorem}
\begin{proof}[Proof Sketch]
Starting from $F = (Ar,att)$, we construct an AF $F_s=(Ar_s,att_s)$ as follows.
For each $a\in Ar$, such that $(a,b)\in att$ for some $b\in Ar$, we consider a new argument $sa$, and a new attack $(sa,b)$.
So that, $Ar_s = Ar\cup \{sa\mid \exists b\in Ar\mbox{, s.t. }(a,b)\in att\}$, and $att_s = att\cup\{(sa,b)\mid \exists b\in Ar\mbox{, s.t. }(a,b)\in att\}$.
Now, we compute the stable extensions of $F_s$.
Then, we consider the minimal stable extensions with respect to the subset inclusion with respect to the arguments in $Ar_s\setminus Ar$, i.e., the set of all $A\in \mathit{stb}(F_s)$ for which there is no $A'\in\mathit{stb}(F_s)$ such that $A'\cap(Ar_s\setminus Ar)$ $\subset$ $A\cap(Ar_s\setminus Ar)$. We denote by $\mathit{mstb}(F_s)$ the set of all minimal stable extensions of $F_s$.
Finally, by filtering the new atoms, we can prove that what we obtain is exactly the set of paracoherent extensions of $F$, i.e., $\mathit{para}(F)=\{A\cap Ar\mid A\in \mathit{mstb}(F_s)\}$.

To conclude the proof, we note that the stable extensions of $F_s$ coincide with the answer sets of $P_{F_s}$, i.e., $\mathit{stb}(F_s)=\AS(P_{F_s})$~\cite{DBLP:journals/ai/Dung95}.
Now, in~\cite{DBLP:conf/aaai/AmendolaD0R18} has been introduced a semantical characterization for semi-stable models in terms of minimal externally supported (MES) models, by replacing the $\kappa$-transformation of $P$ with the so-called externally supported program of $P$.
That is, rules of the form~(\ref{eq:ep-1})-(\ref{eq:ep-4}) are replaced by rule
$a_1 \vee\ldots\vee a_l \leftarrow b_1,\ldots,b_m, \naf c_1,\ldots,\naf c_n,\naf sc_1, \ldots, \naf sc_n$, and then choice rules of the form $\{sc_j\}$ are used to minimize the number of atoms of the form $sc_j$. 
Since we are dealing with logic programs of the form of $P_F$, the rule above becomes 
$a \la \naf c_1,\ldots,\naf c_n,\naf sc_1, \ldots, \naf sc_n$,
that is exactly what the translation from $F$ to $F_s$ has done before.
Therefore, it is easy to check that
$\mathit{MES}(P_F)$ $=$ 
$\{M\in \AS(P_{F_s})\mid 
\nexists M'\in
\AS(P_{F_s})
\mbox{ s.t. } s(M')\subset s(M) \}$ 
$=$ $\mathit{mstb}(F_s)$.
Moreover, by Theorems~3-4 in~\cite{DBLP:conf/aaai/AmendolaD0R18}, we conclude that 
$\mathit{SST}^{\bf t}(P_F) = \mathit{MES}^{\bf t}(P_F)$, where $\mathit{MES}^{\bf t}(P_F)=\{M\cap\Sigma\mid M\in \mathit{MES}(P_F)\}$.
Hence, $\mathit{SST}^{\bf t}(P_F) = \{A\cap Ar\mid A\in \mathit{mstb}(F_s)\} = \mathit{para}(F)$. 
\end{proof}

\begin{example}\label{ex:para=seq}
Consider again the AF $F$ reported in Figure~\ref{fig:stage-semi}. 
The paracoherent answer sets of $P_F$ (see Example~\ref{ex:PF}) are $\{a, c, e, Kf\}$, $\{a, c, g, Ke\}$,
$\{a, d, g, Ke\}$, and $\{a, d, Kf\}$.
Hence, $\mathit{SEQ}^{\bf t}(P_F)=\{\{a,c,e\},$ $\{a,c,d\},\{a,d,g\},\{a,d\}\}$, that is equal to $\mathit{para}(F)$, as shown in the Example~\ref{ex:para-sem-stage}. 
\end{example}

Note that, as it happens in the case of the relationship between the semi-stable and L-stable semantics~\cite{DBLP:journals/ijar/CaminadaSAD15},
if we encode a logic program $P$ into an AF $F_P$ using the association presented in~\cite{DBLP:journals/ijar/CaminadaSAD15}, it is not guaranteed that
the paracoherent extensions of $F_P$ can be mapped one to one with the paracoherent answer sets of $P$.
E.g., if we consider the program $P = \{ b \leftarrow \naf\ a;\ c \leftarrow b, \naf\ c \}$,
the paracoherent answer sets of $P$ are $\{Ka\}$ and $\{b, Kc\}$, whereas  $F_P$ has only one paracoherent extension corresponding to $\{b, Kc\}$.


\section{Computational Complexity}
In this section we study computational complexity issues.
In particular, we prove that the credulous and skeptical reasoning tasks for the paracoherent argumentation semantics remain the same as for semi-stable and stage semantics.


\begin{theorem}
For paracoherent semantics,
credulous reasoning is $\Sigma_2^P$-complete, and
skeptical reasoning is $\Pi_2^P$-complete.
\end{theorem}
\begin{proof}[Proof Sketch]
The memberships are a corollary of the Theorem~\ref{th:para=seq}. Indeed, credulous and skeptical reasoning for paracoherent semantics coincide with brave and cautious reasoning for semi-equilibrium model semantics, resp.; and it is well-known that for semi-equilibrium semantics on normal logic programs,
credulous (brave) reasoning is $\Sigma_2^P$-complete, and
skeptical (cautious) reasoning is $\Pi_2^P$-complete~\cite{DBLP:conf/kr/EiterFM10}.
Concerning the hardness part, it can be checked that the proof of the hardness for credulous (brave) and skeptical (cautious) reasoning for semi-equilibrium model semantics in case of normal logic programs (see Theorem 10 and Appendix C1 in~\cite{DBLP:conf/kr/EiterFM10}) can be directly used to prove that also for normal logic programs with an empty positive body and a single rule for each head atom (as it is $P_F$) hardness results remain unchanged.
\end{proof}

Finally, it is worthy to note that credulous and skeptical reasoning tasks associated with paracoherent argumentation semantics satisfy the criteria identified by~\cite{DBLP:conf/birthday/GagglRT14}, being at the second level of the polynomial hierarchy.

\section{Related Work}\label{sec:related}
In this section we review related works by first mentioning the relation with paracoherent semantics and, then, comparing in detail with alternative argumentation semantics.

\subsection{Paracoherent semantics}
The two major paracoherent semantics for logic programs are the semi-stable~\cite{DBLP:journals/logcom/SakamaI95} and the semi-equilibrium semantics~\cite{DBLP:journals/ai/AmendolaEFLM16}.
These semantics emerged over several alternative proposals~%
\cite{DBLP:journals/ngc/Przymusinski91,DBLP:journals/jacm/GelderRS91,DBLP:conf/lpnmr/SaccaZ91,DBLP:journals/jcss/YouY94,DBLP:journals/amai/EiterLS97,DBLP:conf/lpkr/Seipel97,Balduccini03logicprograms,pere-pint-95,pere-pint-07,DBLP:journals/japll/AlcantaraDP05,DBLP:journals/logcom/GalindoRC08}.
However, \cite{DBLP:journals/ai/AmendolaEFLM16} have shown that only semi-stable semantics~\cite{DBLP:journals/logcom/SakamaI95} and semi-equilibrium semantics~\cite{DBLP:journals/ai/AmendolaEFLM16}  satisfy all the following five highly desirable --from the knowledge representation point of view-- theoretical properties:
(i) every consistent answer set of a program corresponds to a paracoherent answer set (\textit{answer set coverage});
(ii) if a program has some (consistent) answer set, then its paracoherent answer sets correspond to answer sets (\textit{congruence});
(iii) if a program has a classical model, then it has a paracoherent answer set (\textit{classical coherence});
(iv) a minimal set of atoms should be undefined (\textit{minimal undefinedness});
(v) every true atom must be derived from the program (\textit{justifiability or foundedness}).
The first two properties ensure that the notions of answer sets and paracoherent answer sets should coincide for coherent programs; the third states that paracoherent answer sets should exist whenever the program admits a (classical) model; the last two state that the number of undefined atoms should be minimized,
and every true atom should be derived from the program, respectively.
The partial evidential stable models of~\cite{DBLP:conf/lpkr/Seipel97} are known to be equivalent to semi-equilibrium ones~\cite{DBLP:journals/ai/AmendolaEFLM16}. 
An alternative characterization of semi-stable and the semi-equilibrium semantics based on the concept of \textit{externally supported atoms} was given in~\cite{DBLP:conf/aaai/AmendolaD0R18}, that demonstrated to be amenable to obtain efficient implementations.

\subsection{Comparison with Argumentation Semantics}

In this section, we compare paracoherent semantics with alternative semantics for AFs.
First, we compare paracoherent semantics with semi-stable semantics with respect to some basic features of admissible-based semantics. Then, we compare paracoherent semantics with stage semantics by focusing on the behaviour on unattacked arguments; and we give an extensive comparison with many argumentation semantics by focusing on a specific desirable behaviour that paracoherent semantics exhibits, and other semantics do not.
Finally, we provide a discussion concerning the connection with some other studies on the existence of stable extensions.

\subsubsection{Cycles and non-admissibility: Paracoherent Semantics vs Semi-stable Semantics}
Great attention to the problem of loops and cycles has been given both in argumentation \cite{DBLP:journals/logcom/Bench-Capon16,DBLP:journals/logcom/Gabbay16a,DBLP:journals/logcom/DvorakG16,DBLP:journals/logcom/Arieli16,DBLP:journals/logcom/BodanzaTS16}, and in answer set programming~\cite{DBLP:conf/iclp/LeeL03,DBLP:conf/aaai/LinZ04,DBLP:journals/tplp/CostantiniP05,DBLP:journals/tplp/GebserLL11}.
An AF described by an odd-length cycle admits the empty admissible set only. Hence, admissible-based semantics, such as complete, grounded, preferred, semi-stable, stable, ideal~\cite{DBLP:journals/ai/DungMT07}, eager~\cite{conf/bnaic/Caminada07}, and resolution-based grounded~\cite{DBLP:journals/ai/BaroniDG11} can (eventually) admit the empty set as the unique solution. This is the case, in particular, for semi-stable semantics. 
However, in case of AFs described by even-length cycles several extended solution are possible. 
For instance, if $F=(\{a,b,c,d\},$ $\{(a,b),$ $(b,c),(c,d),(d,a)\})$ (the $4$-length cycle), we have two stable extensions, $\{a,c\}$ and $\{b,d\}$, that are also semi-stable extensions.

The paracoherent semantics is non-admissible, unlike semi-stable. 
In many situations, non-admissible semantics allow one to have solutions that are not empty. 
This is a desirable behavior in practice as noted in~\cite{Verheij96twoapproaches}. 
Indeed, it is known that a small initial incoherence might prevents to draw interesting conclusions with the semi-stable semantics. Intuitively, this happens when the argumentation framework ``starts'' with an odd-cycle (i.e., an odd-cycle appears in a strongly connected component of the graph and no argument of this component is attacked by others).

For example, consider the following argumentation framework:
$$F=(\{a,b,c,d,e\}, \{(a,a),(a,b),(b,c),(c,d),(d,e)\}).$$
The starting loop on the node $a$ avoids to have a non-empty admissible solution. Hence, the only semi-stable extension is the empty set.
On the other hand, the paracoherent approach allows to obtain solutions also in case of ``initial incoherence''. 
For instance, in this example, we have $\{c,e\}$ as (the unique) paracoherent solution.

The paracoherent semantics, like others non admissible-based semantics (such as stage), 
guarantees the relevant feature of ensuring a ``symmetric'' treatment of odd- and even-length cycles~\cite{DBLP:journals/ai/BaroniGG05a,DBLP:journals/ai/BaroniG07,DBLP:journals/logcom/DvorakG16}. 
A symmetric behavior is strictly connected with the admissibility property, that guarantees to each argument in a conflict-free solution to be defended by another argument. This cannot happen when we are in the presence of odd cycles. Hence, the unique solution must be the empty one, i.e., we suspend any judgment in the presence of inconsistencies (as in semi-stable). However, in real scenarios this might not suffice. Consider for example a person charged with a crime. In practice, it will be punished or not. A decision will however be made. In such circumstances, we need to keep the argumentation system capable of providing plausible solutions.
For these reasons, we believe that it is necessary to resort to non-admissible semantics. However, to minimize "non-admissibility" it is necessary to keep "as close as possible" to admissible solutions, and, in particular, to maintain a similar/symmetrical behavior to that of stable solutions.

It is not a case that non-admissible semantics can be considered a natural way of evaluating argumentation frameworks. This was evidenced by a recent empirical study concluding that non-admissible semantics {``were the best predictors of human argument evaluation''}~\cite{DBLP:conf/jelia/CramerG19}.

\subsubsection{Unattacked arguments: Paracoherent Semantics vs Stage Semantics}
Paracoherent semantics provides a better behaviour in case of ``unattacked arguments'' (see,~\cite{DBLP:journals/logcom/CaminadaCD12}) when compared with the stage semantics.

For example, consider the following argumentation framework:
$$F=(\{a,b,c,d,e\}, \{(a,b), (b,c),(c,c), (c,d), (d,e)\}),$$ where argument $a$ attacks $b$; $b$ attacks $c$; $c$ attacks itself and $d$; and $d$ attacks $e$. Note that, in particular, $a$ is attacked by no argument. 
The argumentation framework $F$ has two stage extensions, namely $A_1 = \{a,d\}$ and $A_2 = \{b,d\}$. 
Indeed, $A_1\cup A_1^+ =\{a,b,d,e\}$ and $A_2\cup A_2^+=\{b,c,d,e\}$, and it can be easily checked that they are maximal with respect to conflict-free sets.

Now, intuitively, it is very strange that the argument $a$ is false in some extension, namely $A_2$, because $a$ is an argument attacked by no other argument. 
There is no reason to consider $a$ false. 
Hence, one expects to see $a$ as true in each extension, that is $a$ should be a skeptical argument.

This expected behaviour is maintained by paracoherent extensions. Indeed,
it can be checked that $B_1 = \{a,d\}$ and $B_2 = \{a,e\}$ are all the paracoherent extensions. 
Thus, $a$ is a skeptical argument as expected.

\subsubsection{A distinguishing feature of Paracoherent Semantics: The Symmetric Behaviour}

We now provide an extensive comparison between paracoherent semantics and several argumentation semantics. 
We discuss a distinguishing feature of the paracoherent semantics: its symmetric behaviour.

\begin{figure}
\centering 
\subfigure[$5$-radial star polygon]
{
\begin{tikzpicture}[scale=1.0, transform shape]
\def \n {5}
\def \radius {0.8cm}
\def \radiusExt {1.6cm}
\def \margin {7} 
\def \angle {-18}

\node[] (c) at (0:0) {$c$};

\foreach \s in {1,...,\n}
{
  \node[] (b\s) at ({360/\n * (\s )+\angle}:\radius) {$b_\s$};

  \path[->, >=latex] ({360/\n * (\s + 0.5)+\angle+\margin}:\radiusExt) 
    edge ({360/\n * (\s + 1.5)+\angle-\margin}:\radiusExt);

  \node[] (a\s) at ({360/\n * (\s + 0.5)+\angle}:\radiusExt) {$a_\s$};

  \path[->, >=latex]	({360/\n * (\s - 0.5)+\angle}:\radiusExt-0.2cm)  edge ({360/\n * (\s - 0.8)+\angle}:0.2cm+\radius);
  \path[->, >=latex]	({360/\n * (\s - 0.5)+\angle}:\radiusExt-0.2cm)  edge ({360/\n * (\s - 0.2)+\angle}:0.2cm+\radius);

  \path[->, >=latex]	({360/\n * (\s - 1)+\angle}:-0.2cm+\radius)  edge (c);
}
\end{tikzpicture}
}
\hspace*{1.5cm}
\subfigure[$6$-radial star polygon]
{
\begin{tikzpicture}[scale=1.0, transform shape]
\def \n {6}
\def \radius {0.8cm}
\def \radiusExt {1.6cm}
\def \margin {7} 

\node[] (c) at (0:0) {$c$};

\foreach \s in {1,...,\n}
{
  \node[] (b\s) at ({360/\n * (\s )}:\radius) {$b_\s$};

  \path[->, >=latex] ({360/\n * (\s + 0.5)+\margin}:\radiusExt) 
    edge ({360/\n * (\s + 1.5)-\margin}:\radiusExt);

  \node[] (a\s) at ({360/\n * (\s + 0.5)}:\radiusExt) {$a_\s$};

  \path[->, >=latex]	({360/\n * (\s - 0.5)}:\radiusExt-0.2cm)  edge ({360/\n * (\s - 0.8)}:0.2cm+\radius);
  \path[->, >=latex]	({360/\n * (\s - 0.5)}:\radiusExt-0.2cm)  edge ({360/\n * (\s - 0.2)}:0.2cm+\radius);

  \path[->, >=latex]	({360/\n * (\s - 1)}:-0.2cm+\radius)  edge (c);
}
\end{tikzpicture}
}
\caption{A case of symmetric attacks.}
\label{fig:symm}
\end{figure}
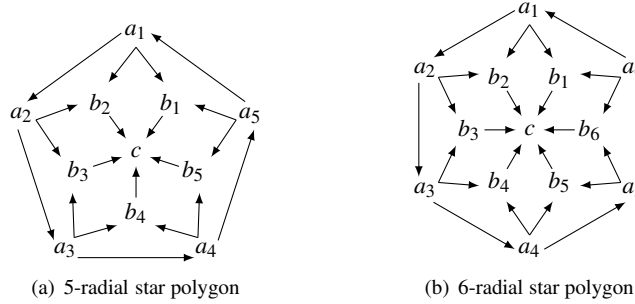

Consider the AF $F_n=(Ar,att)$, where
$Ar=\{a_1,...,a_n,b_1,...,b_n,c\}$, and
$att$ is formed by
the cycle of length $n$ given by $(a_i,a_{i+1})$, for $i=1,...,n-1$, and $(a_n,a_1)$;
the pairs of attacks $(a_i,b_i)$, $(a_i,b_{i+1})$, for  $i=1,...,n-1$, and
$(a_n,b_n)$, $(a_n,b_1)$;
and finally $(b_i,c)$, for $i=1,...,n$.
Intuitively, $a_1$, ..., $a_n$ form a cycle of attacks; each $a_i$ attacks two arguments subscripted as consecutive, $b_i$ and $b_{i+1}$, except for the last $a_n$ which attacks $b_n$ and $b_1$, by completing a sort of cyclic attack; and each $b_i$ attacks the argument $c$. 
We will call this graph the $n$-\textit{radial star polygon}.
Figure~\ref{fig:symm} reports what is happening in case of $n=5$ and $n=6$.

In case of $n$ is even, we obtain two stable extensions:
$\{a_1,a_3,...,a_{n-1},c\}$ and 
$\{a_2,a_4,...,a_{n},c\}$. Note that $c$, in particular, is always inferred, so it is a skeptical argument.
However, as well-known, in case of $n$ is odd, the unique admissible set is the empty one, whereas no stable extension exists. We stress that no admissible-based semantics can have a symmetric behaviour with respect to cycles.
However, also non admissible-based semantics have this issue with cycles. This happens, for instance, for 
$cf2$~\cite{DBLP:journals/ai/BaroniGG05a,DBLP:journals/logcom/GagglW13} argumentation semantics, as shown by~\cite{DBLP:journals/logcom/DvorakG16}. In our symmetric example, for $cf2$, the extensions of $F_5$ are $\{a_1,a_3,b_5\}$, $\{a_2,a_4,b_1\}$, $\{a_3,a_5,b_2\}$, $\{a_4,a_1,b_3\}$, and $\{a_5,a_2,b_4\}$, thus $c$ does not belong to any of them.
Nonetheless, among the $cf2$ extensions of $F_6$ there are $\{a_1,a_3,a_5,c\},$ $\{a_2,a_4,a_6,c\},$ $\{a_1,a_4,b_3,b_6\},$ $\{a_2,a_5,b_4,b_1\},$ $\{a_3,a_6,b_5,b_2\}$, hence $c$ is not skeptically but credulously accepted, thus resulting in a different behavior for even and odd radial star polygons.

This non-symmetric behaviour happens also for the stage semantics. 
Since the stage semantics is a generalization of the stable semantics, and, whenever a stable extension exists, $c$ is a skeptical argument, then, to fulfill a symmetric behaviour, the stage semantics should guarantee the skeptical acceptance of $c$.
However, just considering $n=3$, the stage extensions will be
$\{a_1,b_3\}$, $\{a_2,b_1\}$, and $\{a_3,b_2\}$. Thus, $c$ is not even a credulous argument. 
Note that this holds also for extensions of stage semantics such as the \textit{stagle}~\cite{DBLP:conf/comma/BaumannLW16}.

Therefore,
admissible, complete, grounded, ideal, preferred, stable, semi-stable, eager, resolution-based grounded, stage, and cf2  semantics do not have a symmetric behaviour w.r.t. radial star polygons.
While, paracoherent semantics on $n$-radial star polygon infers $c$ as skeptical argument, for each $n$.
In particular, whenever $n$ is odd, the set of the paracoherent extensions of $F_n$ is formed by $A_i=\{c\}$ $\cup$ $\{a_{i-2h\hspace*{-0.1cm}\pmod n}: h=0,1,...,\frac{n-3}{2}\}$, for each $i=1,...,n$. 

\subsubsection{Further studies  on the non-existence of stable extensions}

As a final observation, we note that the reasons for non-existence of stable extensions has been recently investigated in~\cite{DBLP:journals/ai/SchulzT18}.
In particular, the idea of \cite{DBLP:journals/ai/SchulzT18} is to ``fix'' preferred extensions that are not stable by applying a structural revision of the original AF. This revision can be used to make stable some preferred extensions, no matter whether the AF admits a stable one. 
This goal is rather different from the ideas underlying the paracoherent semantics which aims at finding a (minimal) remedy to missing stable extensions without modifying the original AF.
Nonetheless, an interesting open question is whether stabilizers are in some way related to the concept of responsible sets of~\cite{DBLP:journals/ai/SchulzT18}.


\section{Conclusion}
This paper introduces a different perspective on AFs with no stable extension, by proposing the paracoherent extensions.
The new semantics coincides with the stable semantics, whenever a stable extension exists, and has a natural counterpart in paracoherent semantics for logic programs.
Moreover, we studied the computational complexity  of the main reasoning tasks, that remain unchanged in comparison with semi-stable and stage semantics.
Finally, we compared paracoherent semantics with several existing argumentation semantics, by showing an interesting distinctive symmetric behaviour on graphs that involve odd-length cycles.
In the literature of Argumentation frameworks several semantics have been proposed, each one having some distinctive feature.
It is difficult, and probably impossible, to identify an overall winner in this context, where often proposals are incomparable.
However, it can be observed that symmetry is very common property in nature, and has been often subject of cross-disciplinary studies. 
Already our running example demonstrates the desirable symmetric behaviour of our semantics is useful while modelling 
a very well known problem (Stable Roommates) with a large number of applications in real world~\cite{marriage}.
Our contribution is also relevant because it adds a missing link in the panorama of correspondences of AF semantics with logic programming ones.

As future work, we will deepen the study of the symmetrical behavior and to study additional properties of our semantics as done in~\cite{DBLP:journals/ai/BaroniG07}. 
Moreover, we plan to implement paracoherent semantics to solve typical reasoning problems considered in competitions. 

\bibliographystyle{acmtrans}
\bibliography{biblio}

\label{lastpage}
\end{document}